\pdfoutput=1
\documentclass{article}
\usepackage{arxiv}

\usepackage{graphicx} % Required for inserting images
\usepackage{amsmath}
\usepackage{amsthm}
\usepackage{amssymb}
\usepackage{url}
\usepackage{hyperref}

\usepackage[numbers]{natbib}
\usepackage{paralist}
\usepackage{wrapfig}

\usepackage{tikz}
\usetikzlibrary{graphs}
\usetikzlibrary{arrows, arrows.meta}
\tikzset{>={[round,sep]Stealth}}

\usetikzlibrary{decorations}
\pgfdeclaredecoration{simple line}{initial}{
  \state{initial}[width=\pgfdecoratedpathlength-1sp]{\pgfmoveto{\pgfpointorigin}}
  \state{final}{\pgflineto{\pgfpointorigin}}
}
\tikzset{
   shift left/.style={decorate,decoration={simple line,raise=#1}},
   shift right/.style={decorate,decoration={simple line,raise=-1*#1}},
}
\newtheorem{theorem}{Theorem}
\newtheorem{lemma}[theorem]{Lemma}
\newtheorem{observation}[theorem]{Observation}
\newtheorem{remark}[theorem]{Remark}

\newtheorem{definition}{Definition}
\newtheorem{corollary}{Corollary}

\newcommand{\IR}{\mathbb{R}}
\newcommand{\existsR}{\ensuremath{\mathsf{\exists\IR}}}
\newcommand{\forallR}{\ensuremath{\mathsf{\forall\IR}}}
\newcommand{\ETR}{\ensuremath{\mathrm{ETR}}}
\newcommand{\QUAD}{\ensuremath{\mathrm{QUAD}}}
\newcommand{\ETRpp}{\ensuremath{\mathrm{ETR}^{++}}}
\newcommand{\QUADpp}{\ensuremath{\mathrm{QUAD}^{++}}}
\newcommand{\DIM}{\ensuremath{\mathrm{DIM}}}
\newcommand{\FFEAS}{\ensuremath{4\mathrm{FEAS}}}
\newcommand{\threeSAT}{\ensuremath{\mathrm{3SAT}}}

\newcommand{\NP}{\ensuremath{\mathsf{NP}}}
\newcommand{\coNP}{\ensuremath{\mathsf{coNP}}}
\newcommand{\PSPACE}{\ensuremath{\mathsf{PSPACE}}}

\newcommand{\PD}{\mathrm{PD}}
\newcommand{\fib}{\mathcal{F}}
\newcommand{\im}{\operatorname{im}}

\title{On the Complexity of Identification in \\ Linear Structural Causal Models}

\usepackage{authblk}

\begin{document}

\author[1]{Julian D\"{o}rfler\thanks{Contributing equally first authors.}\, }%{Saarland University, Germany}%{jdoerfler@cs.uni-saarland.de}{}{}
\author[2]{Benito van der Zander$^{*\,}$}%{University of Lübeck, Germany}%{b.vanderzander@uni-luebeck.de}{}{}
\author[1]{Markus Bl\"{a}ser\thanks{Contributing equally last authors.}\, }%{Saarland University, Germany}%{mblaeser@cs.uni-saarland.de}{}{}
\author[2]{Maciej Li\'{s}kiewicz$^{\dagger\,}$}%{University of Lübeck, Germany}%{maciej.liskiewicz@uni-luebeck.de}{}{}

\affil[{ }]{\footnotesize\{mblaeser,jdoerfler\}@cs.uni-saarland.de; \{maciej.liskiewicz,b.vanderzander\}@uni-luebeck.de}
\affil[1]{\footnotesize Saarland University, Germany}
\affil[2]{\footnotesize University of Lübeck, Germany}

\maketitle

\begin{abstract}
Learning the unknown causal parameters of a linear structural causal 
model is a fundamental task in causal analysis. The task, known as the 
problem of identification, asks to estimate the parameters of the model from a
combination of assumptions on the graphical structure of the model and 
observational data, represented as a non-causal covariance matrix.
In this paper, we give a new sound and complete algorithm for generic 
identification which runs in polynomial space. By standard simulation 
results, this algorithm has exponential running time which vastly improves 
the state-of-the-art double exponential time method using a Gr\"obner basis 
approach. The paper also presents evidence that parameter identification 
is computationally hard in general. In particular, we prove, that the task
asking whether, for a given feasible correlation matrix, there 
are exactly one or two or more parameter sets explaining the observed 
matrix, is hard for $\forallR$, the co-class of the existential theory 
of the reals. In particular, this problem is $\mathsf{co}\NP$-hard.
To our best knowledge, this is the first hardness result for some notion 
of identifiability.
\end{abstract}

\section{Introduction}
Recognizing and predicting the causal effects and distinguishing
them from purely statistical correlations is an important task of empirical sciences. 
For example, identifying the causes of disease and health outcomes is of great significance 
in developing new disease prevention and treatment strategies.
A common approach for establishing causal effects is through randomized controlled trials 
(Fisher, \citep{fisher1936design}) -- called the gold standard of experimentation -- which, however, 
requires physical intervention in the examined system. Therefore, in many practical applications
experimentation is not always possible due to cost, ethical constraints, or technical feasibility.
E.g., to learn the effects of radiation on human health one cannot conduct interventional studies 
involving human participants. In such cases, a researcher may use an alternative approach and
establish cause-effect relationships by combining existing observed data with the knowledge of 
the structure of the system under study. 
This is called the problem of \emph{identification}
in causal inference (Pearl, \citep{Pearl2009}) and the approach is commonly used in various
fields, including modern ML.

A key ingredient of this framework is the way the underlying structure models the true mechanism behind the system. In general, this is done using structural causal models (SCMs) \citep{Pearl2009,bareinboim2022pearl}.
In this work, we focus on the problem of identification in linear SCMs, known also as structural 
equation models (SEMs) \citep{bollen2014structural,duncan2014introduction}.
They represent the causal relationships between observed 
random variables assuming that each variable 
$X_i$, with $i=1,\ldots, n$ is linearly dependent on  the remaining variables and 
an unobserved error term $\varepsilon_i$ of normal distribution with zero mean:
$X_j=\sum_{i} \lambda_{i,j} X_i +\varepsilon_i$. The model implies the existence of 
some  covariance matrix $\Omega=(\omega_{i,j})$ between the error terms.
In this paper, we consider  \emph{recursive models}, i.e., we assume that, for all $i>j$, we have $\lambda_{i,j} =0$.

Linear SCMs can be represented as a graph with nodes  $\{1,\ldots,n\}$ and with directed and bidirected edges. 
A directed edge $i\to j$ represents a linear influence $\lambda_{i,j}\not= 0$ of a parent node 
$i$ on its child $j$. A bidirected edge $i \leftrightarrow j$ represents a correlation 
$\omega_{i,j}\neq 0$ between error terms $\varepsilon_i$ and $\varepsilon_j$ (cf.~Figure~\ref{fig:IV}). 

\begin{wrapfigure}{R}{0.28\textwidth}
\vspace*{-4mm}
	\begin{tikzpicture}[scale=0.9]
	\footnotesize
    	\node[circle, draw, scale=0.9] (z) at (0,0) {1};
    	\node[circle, draw, scale=0.9] (x) at (1.5,0) {2};
    	\node[circle, draw, scale=0.9] (y) at (3,0) {3};
    	\draw [->] (z) -- node [midway,below] {$\lambda_{1,2}$} (x);
    	\draw [->] (x) -- node [midway,below] {$\lambda_{2,3}$} (y);
    	\draw [<->] (x) edge [bend left=45] node [midway,above]{$\omega_{2,3}$}(y);
  \end{tikzpicture}
\caption{An IV example. \label{fig:IV}}
\end{wrapfigure}
Writing the coefficients of all directed edges as adjacency matrix 
$\Lambda=(\lambda_{i,j})$ and the coefficients of all bidirected edges as adjacency 
matrix $\Omega=(\omega_{i,j})$, the covariances $\sigma_{i,j}$ between each pair 
of observed variables $X_i$ and $X_j$ can be calculated as matrix $\Sigma = (\sigma_{i,j})$:
\begin{equation}
\Sigma = (I - \Lambda)^{-T} \Omega (I-\Lambda)^{-1}, \label{eqn:SEM}
\end{equation}
where $I$ is the identity matrix \cite{foygel2012half}.  
Knowledge of the parameters $\lambda_{i,j}$ allows for the prediction of 
all causal effects in the system and the key task here is to learn $\Lambda$ 
assuming  $\Omega$ remains unknown. This leads to the formulation of the 
identification problem in linear SCMs as solving for the parameter $\Lambda$ 
using Eq.~\eqref{eqn:SEM}. If the problem asks to find symbolic solutions 
involving only symbols $\sigma_{i,j}$, we call it the problem of 
\emph{symbolic identification}. In the case when the parameter can be 
determined uniquely almost everywhere using alone $\Sigma$, 
we call the instance to be \emph{generically identifiable}
(for a formal definition, see Sec.~\ref{sec:preliminaries}).
If the goal is to find, for a given $\Sigma$ of rational numbers, 
the numerical solutions, we call it the problem of \emph{numerical identification}.
In this paper, we study the computational complexity of both variants of the problem.

{\bf Previous Work.} 
The identification in linear SCMs and its applications has been 
a subject of research interest for many decades, including 
the early work in econometrics and agricultural sciences
\citep{wright1921correlation,wright1928tariff,fisher1966identification,bowden1984instrumental}.
Currently, it seems, that one of the most challenging tasks in this field
is providing effective computational methods 
that allow automatic finding of solutions, both for symbolic and for numeric variants,
or to provide evidence that the problems are computationally intractable.

The generic identification can be established using standard algebraic 
tools for solving symbolic polynomial equations~\eqref{eqn:SEM}. Such an approach 
provides a \emph{sound} and \emph{complete} method, i.e., it is guaranteed to identify
all identifiable instances. However, common algorithms for solving 
such equations usually use Gröbner basis computation, whose time complexity 
is doubly exponential in the worst case \citep{garcia2010identifying}. 
%Therefore, they are often too slow even for cases of very small sizees 
So far, it has remained widely open whether the  double exponential function
is a sharp upper bound on the computational complexity of the generic identifiability.

Most approaches to solving the problem in practice are based on instrumental variables, 
in which the causal direct effect is identified as a fraction of two covariances 
\citep{wright1928tariff,bowden1984instrumental}. For example, in the linear model
shown in Figure~\ref{fig:IV}, one can calculate first $\lambda_{1,2} = \sigma_{1,2}$
and then $\lambda_{2,3} = \frac{\lambda_{1,2}\lambda_{2,3}}{\lambda_{1,2}} = \frac{\sigma_{1,3}}{\sigma_{1,2}}$.  
The variable $X_1$ is then called an instrumental variable (IV). 
This method is sound but not complete, that is, when it identifies a parameter,
then it is always correct. However, when the method fails due to a missing IV, 
then the parameter can still be identified by other means.
This approach has inspired intensive research aimed at providing
computational methods that may not be complete 
but enable effective algorithms and identify a 
significantly large number of cases.

To the most natural extensions of the simple IVs belong 
conditional IVs (cIV) \citep{bowden1984instrumental,PearlConditionalIV}.
The corresponding identification method is based on an efficient, 
polynomial time algorithm for finding conditional IVs \citep{IJCAIInstruments}.
More complex criteria and methods proposed in the literature, which 
are also accompanied by polynomial time algorithms, involve
instrumental sets (IS) \citep{BritoPearlUAI02}
half-treks (HTC) \citep{foygel2012half},
instrumental cutsets (IC) \citep{kumor2019instrumentalCutsets},
auxiliary instrumental variables (aIV) \citep{chen2015incorporating}, 
The generalized HTC (gHTC) 
\citep{chen2016identification,identifyingEdgeWiseDeterminantalDrton}
and auxiliary variables (AVS) 
\citep{chen2016identification,chen2017identification}  
can be implemented in polynomial time 
provided that the number of incoming edges to each
node in the causal graph is bounded.
The methods based on generalized IS (gIS),
a simplified version of the criterion (scIS), and generalized AVS (gAVS) 
 appeared  to be computationally intractable
\citep{BritoPearlUAI02,brito2010instrumental,brito2002graphical,instrumentalVariablesZander2016,chen2017identification}.
The auxiliary cutsets (ACID) algorithm \citep{kumor2020auxiliaryCutsets} 
subsumes all the above methods in the sense that
it covers all the instances identified by them.
Recently  \citep{TreeID2022vanderzander,TreeID2024gupta} provide
the TreeID algorithm for 
identification in tree-shaped linear models.
TreeID is incomparable since it is complete for the subclass of tree-like SCMs. However, TreeID does not even work for general SCMs, which is the focus of our work.
Figure~\ref{fig:identification:sem:methods} summarizes these results.

\begin{figure} %[h]
\tikzset{semFast/.style={draw=green!70!black,very thick}}
\tikzset{semSlow/.style={draw=red!80!black,very thick}}
\tikzset{semComplete/.style={draw=blue!60!black,very thick}}
\begin{tikzpicture}[xscale=4,every node/.style={draw},yscale=1.2,xscale=0.9]
\footnotesize
\node[semFast] (IV) at (0.2,0) {IV \citep{wright1928tariff}};
\node[semSlow] (gIS) at (1, 0.7) {gIS, scIS \citep{brito2002graphical,instrumentalVariablesZander2016}};
\node[semFast,text width=5.5cm] (HTC) at (1.3,0) {IS, HTC, AVS, gHTC \citep{BritoPearlUAI02,foygel2012half,chen2017identification,chen2016identification,identifyingEdgeWiseDeterminantalDrton} };
\node[semFast] (cIV) at (1,-0.7) {cIV, aIV \citep{bowden1984instrumental,chen2015incorporating}};
\node[semSlow] (qAVS) at (2,0.7) {gAVS \citep{chen2017identification}};
\node[semFast] (IC) at (2.35,0) {IC \citep{kumor2019instrumentalCutsets}};
\node[semFast] (cAV) at (1.95,-0.7) {cAV \citep{chen2017identification}};
\node[semFast] (treeID) at (2.7,-0.7) {TreeID \citep{TreeID2022vanderzander,TreeID2024gupta}};
\node[semFast] (ACID) at (2.85,0) {ACID \citep{kumor2020auxiliaryCutsets}};
\draw[->] (IV)+(0.165,0.125) --  (0.663,0.7); %\draw[->] (IV) edge (gIS) 
\draw[->]  (IV) edge (HTC);
 \draw[->]  (IV)+(0.165,-0.125) --  (0.7,-0.7); %(cIV);
\draw[->] (gIS) edge (qAVS) 
          (HTC) edge (IC) 
          (cIV) edge (cAV)
          (HTC) edge (qAVS);
\draw[->] (qAVS) edge (ACID) 
          (IC)   edge (ACID) 
          (cAV)  edge (ACID);
\node[semComplete,text width=2.5cm] (G) at (3.55,0) 
{
Complete Methods\\
{\tiny $\bullet$} double-EXP\\
 \quad\quad Gr\"{o}bner \citep{garcia2010identifying}\\
{\tiny $\bullet$}  PSPACE \\ \quad\quad this paper
};
\draw[->] (ACID) edge (G) ;
\draw[->] (3.01,-0.7) -- (G) ; %(treeID) edge (G) ;
\end{tikzpicture}
\caption{Methods for generic identification in linear SCMs.
An arrow from methods $A \to B$ means $B$ subsume methods $A$, 
i.e., any instance that can be identified by any of methods $A$ can be 
identified by method $B$ and this inclusion is proper.  
Green boxes mean there exist polynomial-time 
algorithms to apply the method, a red box means no such algorithm 
is known or the method has been proven to be NP-hard. 
The blue box includes the complete methods.
}\label{fig:identification:sem:methods}
\end{figure}
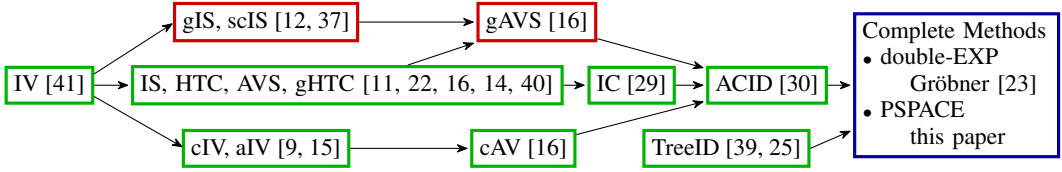

Numerical parameter identification, known in the literature as 
SEM model parameter estimation, has been the subject of a considerable 
amount of research, which has resulted in significant progress in theoretical 
understanding and development of estimation methods \citep{anderson1956statistical,joreskog1969general,browne1974generalized,muthen1983latent,bollen1989structural,hu1992can}.
Currently, in practical applications (e.g. in econometrics, psychometrics or biometrics),
methods based on maximum likelihood (ML)
%maximum likelihood mean adjusted (MLM), 
or generalized least squares estimator (GLS) are commonly used to find model parameters. However, despite the great importance of this problem and considerable effort 
in method development, the computational complexity of the parameter estimation 
problem remains unexplored. In our work, we provide, to our knowledge, the first 
hardness result for a very basic variant of the SEM parameter estimation problem, which 
we call numerical identification.

{\bf Our Contribution.} We improve significantly the best-known 
upper bound on the computational complexity for generic identification 
and provide evidence that parameter identification is computationally 
hard in general. In more detail, our contributions are as follows:

{\tiny $^\bullet$}  We provide a polynomial-space algorithm for sound and complete 
    	generic parameter identification in linear models. This gives an exponential 
	running time which vastly improves the state-of-the-art 
    	double exponential time method using Gröbner basis.
	In our approach, we formulate generic identifiability as a formula over real variables with both existential and universal quantifiers and then use Renegar's algorithm \cite{renegar1992computational}. \\
{\tiny $^\bullet$}  Our constructive technique allows us to prove an $\exists\forall\IR$ (and $\forall\exists\IR$) upper bound 
	on generic identifiability, for the well-studied complexity class $\exists\forall\IR$ (see Sec.~\ref{sec:real} for definitions). It is 
	an intermediate class between $\NP$ and $\PSPACE$.\\
 %and defined as the closure of the ETR under polynomial time many-one reductions.\\
{\tiny $^\bullet$}  We prove that numerical identification is hard for the complexity class 
  	$\forallR$. 
	In particular, the problem is $\mathsf{co}\NP$-hard. 
	Our~complexity characterization is quite precise since we show a 
	(promise) $\forallR$ upper bound for numerical identification.
	To the best of our knowledge, 
 	this is the first hardness result for some notion of identifiability. \\
{\tiny $^\bullet$}  On the other hand, we show that numerical identifiability can 
	be decided in polynomial space.\\
 {\tiny $^\bullet$} 	
 	If an instance is non-identifiable, then an important task is to identify 
	as many model parameters as possible. We are particularly interested 
	in identifying the parameters of specific edges in the graph representing the linear model. 
	In the paper, we obtain, for ``edge identifiability'',  the same results as 
	for the common identifiability problem. Since these proofs are essentially the same, they can be found in Appendix~\ref{sec:edge}.

\section{Preliminaries}\label{sec:preliminaries}

\subsection{The Problems of Identification in Linear Causal Models} \label{sec:ident}

A mixed graph is a triple $G = (V,D,B)$ where $V:= \{1,\ldots, n\}$ is a finite set of nodes and 
$D \subseteq V \times V$ and $B \subseteq \binom V2$ are two sets of directed and bidirected edges, respectively.  
Let $\IR^D$ be the set of matrices $\Lambda=(\lambda_{i,j}) \in \IR^{n\times n}$ with
$ \lambda_{i,j} = 0$ if $i\to j$ is not in $D$ and 
let $\PD(n)$ denote the cone of positive definite $n\times n$ matrices. Let
$\PD(B)$ be the set of matrices $\Omega=(\omega_{i,j}) \in \PD(n)$ 
with $\omega_{i,j}=0$ if $i \not= j$ and $i \leftrightarrow j$ is not an edge in $B$.
In this work, we will only consider recursive models, i.e. we assume that, 
for all $i>j$ we have $\lambda_{i,j}= 0$.
Thus, the directed graph $(V,D)$  accompanied  with the model
is acyclic. In the paper, we will assume w.l.o.g. that the nodes are topologically sorted, i.e., 
 there are no edges $i \to j$ with $i > j$.

Denote by ${\cal N}_n(\mu,\Sigma)$ the multivariate normal distribution with 
mean $\mu\in \IR^n$
and covariance matrix $\Sigma$.
The linear SCMs ${\cal M}(G)$ associated with
$G = (V,D,B)$ is the family of multivariate normal distributions
${\cal N}_n(0,\Sigma)$
with $\Sigma$ satisfying Eq.~\eqref{eqn:SEM},
for $\Lambda\in \IR^D$ and $ \Omega\in \PD(B)$.
A  model in ${\cal M}(G)$
is specified in a natural way in terms of a system of linear structural equations:
$ X_j=\sum_{i\in\text{pa}(j)} \lambda_{i,j} X_i +\varepsilon_i,$
for  $j=1,\ldots, n,$
where $\text{pa}(j)$ denote the parents of $j$ in $G$.
If $\varepsilon =(\varepsilon_1, \ldots,\varepsilon _n)$ is a random vector with the multivariate normal distribution
${\cal N}_n(0,\Sigma)$ and $\Lambda\in \IR^D$,
then the random vector $X=(X_1,\ldots,X_n)$ is well
defined as a solution to the equation system 
and follows a centered multivariate normal distribution with covariance 
matrix $(I - \Lambda)^{-T}  \Omega(I - \Lambda)^{-1}$ (see, e.g.~\cite{drton2011global}).

For a given (acyclic) mixed graph  $G = (V,D,B)$ define the parametrization map
$$
   \phi_G:(\Lambda,\Omega) \mapsto (I - \Lambda)^{-T}  \Omega(I - \Lambda)^{-1}
$$
and let $\Theta:=\IR^D\times \PD(B)$. We define that $G$ is 
 \emph{globally} identifiable if  $\phi_G$ is injective on $\Theta$ (\cite{drton2011global}).

Global identification can be decided easily, see \cite[Theorem 2]{drton2011global}. 
%We call a directed graph with at least two nodes an 
%\emph{arborescence converging} to
%node $i$ if its edges form a spanning tree with a directed path from any node $j \not= i$ to $i$.
%
%\begin{theorem}[\cite{drton2011global}, Theorem 2]
%The parametrization  $\phi_G$ for an acyclic mixed graph $G = (V,D,B)$ 
%fails to be injective if and only if there is an induced subgraph $G_A$,
%$A \subseteq V$, whose directed part $(A,D_A)$ contains a converging arborescence and
%whose bidirected part $(A,B_A)$ is connected. If $\phi_G$ is injective, then its inverse
%is a rational map.
%\end{theorem}
%
%\begin{theorem}[[\cite{drton2011global}, Corollary 1]
%If the acyclic mixed graph $G$ has at most three nodes, then
% $\phi_G$ is injective if and only if $G$ is simple (bow-free), i.e., there is at most one edge
%between any pair of nodes in $G$. There are exactly two unlabeled simple
%acyclic mixed graphs on four nodes with  $\phi_G$  not injective.
%\end{theorem}
However, it is a very strong property. For instance, as seen in the introduction, in Figure~\ref{fig:IV}, we can recover the parameter 
$\lambda_{2,3}$ as  
$\frac{\sigma_{1,3}}{\sigma_{1,2}}$.
So identification fails only in the (very unlikely case) that $\sigma_{1,2} = 0$.
This leads to the concept of \emph{generic identifiability}

\begin{definition}[Generic Identifiability, \cite{foygel2012half}]
The mixed graph $G$ is said to be \emph{generically} identifiable if
 $\phi_G$ is injective on the complement $\Theta \setminus {\cal V}$ of a proper (i.e., strict) 
 algebraic subset ${\cal V}\subset \Theta$.
\end{definition}

Given matries $\Lambda_0 \in \IR^D$ and $\Omega_0 \in \PD(B)$, the corresponding
\emph{fiber} is defined by
\[
   \fib_G(\Lambda_0,\Omega_0) = \{ (\Lambda,\Omega) \mid 
   \phi_G(\Lambda,\Omega) = \phi_G(\Lambda_0,\Omega_0),  
   \Lambda \in \IR^D, \Omega \in \PD(B) \}.
\]
A fiber contains all pairs of matrices that induce the same observed covariance matrix $\Sigma$.
For $\Sigma \in \im \phi_G$, we also write $\fib_G(\Sigma)$ for the fiber belonging to $\Sigma$.
We can phrase identifiabilty in terms of fibers:
\begin{itemize}
    \item $G$ is globally identifiable, if $|\fib_G(\Lambda_0,\Omega_0)| = 1$ for all
    $\Lambda_0 \in \IR^D$ and $\Omega_0 \in \PD(B)$.
    \item $G$ is generically identifiable, if $|\fib_G(\Lambda_0,\Omega_0)| = 1$ for Zariski almost all
    $\Lambda_0 \in \IR^D$ and $\Omega_0 \in \PD(B)$.
\end{itemize}

Generic identifiability asks whether all parameters are almost always identifiable.
It is also of interest to ask whether a single parameter $\lambda_{ij}$ is almost always identifiable.
For this, we consider the projection of the fiber on the single parameter, which we will also call an \emph{edge fiber}:
\[
   \fib^{i,j}_G(\Lambda_0,\Omega_0)
   %= \{ \Lambda_{(ij)} \mid \phi_G(\Lambda,\Omega) = \phi_G(\Lambda_0,\Omega_0), \Lambda \in \IR^D, \Omega \in \PD(B) \}
   = \{ \Lambda_{i,j} \mid (\Lambda, \Omega) \in \fib_G(\Lambda_0,\Omega_0) \}.
\]
\begin{itemize}
    \item $\lambda_{ij}$ is generically edge identifiable, if $|\fib^{i,j}_G(\Lambda_0,\Omega_0)| = 1$ for Zariski almost all
    $\Lambda_0 \in \IR^D$ and $\Omega_0 \in \PD(B)$.
\end{itemize}

For edge identifiability, we only consider parameters $\lambda_{i,j}$ since parameters $\omega_{k,l}$ can be recovered from parameters $\lambda_{i,j}$ \citep{drton2018algebraic}.

Global and generic identifiability are properties of the given mixed graph. In this work,
we also study identification as a property of the observed numerical data, i.e.,
of the observed covariance matrix $\Sigma$.

\begin{definition}[Numerical Identifiability]
Given an acyclic mixed graph $G = (V,D,B)$ and a feasible matrix $\Sigma$, decide whether the parameters are uniquely identifiable, i.e.~if $|\fib_G(\Sigma)|=1$?
\end{definition}
    
Note that this is a promise problem. We assume that $\Sigma$ is feasible, i.e., in the image of $\phi_G$. Therefore, we shall also study the feasibility problem: Given $\Sigma$, is it contained in $\im \phi_G$?

Similarly we can also define numerical edge identifiability: For a given feasible $\Sigma$, test whether the edge fiber $\Sigma$ belongs to has size $1$ or $> 1$.

\subsection{The (Existential) Theory of the Reals}
\label{sec:real}

The existential theory of the reals ($\ETR$) is the set of true sentences
of the form
\begin{equation} \label{eq:etr:1}
   \exists x_1 \dots \exists x_n \ \varphi(x_1,\dots,x_n),
\end{equation}
where $\varphi$ is a quantifier-free Boolean formula over the basis $\{\vee, \wedge, \neg\}$
and a signature consisting of the constants $0$ and $1$, the functional symbols
$+$ and $\cdot$, and the relational symbols $<$, $\le$, and $=$. The sentence
is interpreted over the real numbers in the standard way.
The theory forms its own complexity class $\existsR$ which is 
defined as the closure of $\ETR$ under polynomial-time many-one reductions.
Many natural problems have been shown to be complete for $\ETR$, for instance
the computation of Nash equilibria \citep{DBLP:journals/mst/SchaeferS17}, the famous
art gallery problem \citep{abrahamsen2018art}, or training neural networks \citep{bertschinger2023training}, just to mention a few.

It turns out that one can simplify the form of an $\ETR$-instance. We can get rid
of the relations $<$ and $\le$ and it is sufficient to consider only Boolean conjunctions. 
More precisely, the following problem is $\existsR$-complete: Given polynomials $p_1,\dots,p_m$ in variables
$x_1,\dots,x_n$, decide whether there is a $\xi \in \IR^n$ such that 
\begin{equation} \label{eq:Poly}
  p_1(\xi) = \dots = p_m(\xi) = 0.
\end{equation}
By Tseitin's trick, we can assume that all polynomials are of one of the forms
\begin{equation} \label{eq:tseitin}
  ab - c, \enspace a + b - c, \enspace a - b, \enspace a - 1, \enspace a
\end{equation}
and all variables in each of the polynomials are distinct. Note that all polynomials
in (\ref{eq:tseitin}) have degree at most two. Therefore, this problem
is also called the feasibility problem of quadratic equations $\QUAD$. 
For a proof, see e.g.\ \cite{DBLP:journals/mst/SchaeferS17}.

\paragraph{Universal Quantification.}
If, instead of considering existentially quantified true sentences, we consider universally quantified true sentences of the form
\begin{equation} \label{eq:foralltr:1}
   \forall x_1 \dots \forall x_n \varphi(x_1,\dots,x_n),
\end{equation}
where $\varphi$ is again a quantifier-free Boolean formula, and form the closure under polynomial-time many-one reductions, we obtain the complexity class $\forallR$.
Using De Morgan's law, it is easy to see the well-known fact that $\forallR = \mathsf{co\text{-}\existsR}$, i.e.\ it is the complement class of $\existsR$.

It is also possible to alternate quantifiers, giving rise to a whole hierarchy, see \cite{DBLP:journals/mst/SchaeferS24}. We call the corresponding classes 
$\exists\forall\IR$, $\forall\exists\IR$, \dots

\paragraph{Complexity of $\existsR$ and $\forallR$.}
It is easy to see that quantification over real variables can be used to simulate quantification over Boolean variables by adding the constraint $x(x-1) = 0$.
This way we can convert $\threeSAT$-formulas to $\ETR$-formulas, proving the well-known containment $\NP \subseteq \existsR$.

With his celebrated result about quantifier elimination, Renegar \cite{renegar1992computational} proved that the truth of any sentence over the reals with a constant amount of quantifier alternations is decidable in $\PSPACE$.
This in particular implies
\begin{align}
    \NP \subseteq \existsR &\subseteq \PSPACE & \coNP \subseteq \forallR &\subseteq \PSPACE
\end{align}
While all these inclusions are believed to be strict, it is unknown for all of them.

\section{Finding another solution}

Numerical identification is a promise problem, we assume that the given input is feasible. 
Being a promise problem means that an algorithm for numerical identification should output the correct answer whenever the input is feasible. But it can output anything when the input is not feasible. We give some further information about promise problems in Appendix~\ref{sec:promise} for the reader's convenience.

For our hardness proof, we need to look at instances of $\ETR$ or $\QUAD$ that are satisfiable. Of course, deciding whether a satisfiable instance is satisfiable is a trivial task. So the task will be to decide whether the satisfiable instance has another solution. We call the corresponding
promise problems $\ETRpp$ and $\QUADpp$.

It turns out that these promise problems are $\existsR$-hard. Since $\QUADpp$
is a special case of $\ETRpp$, it suffices to prove this for $\QUADpp$.
Let $y$ be an extra variable. We will plant an extra solution into the 
system~(\ref{eq:Poly}):
\begin{align}
    y (y-1) & = 0  \label{eq:pp:1} \\
    y x_i   & = 0 \qquad i = 1,\dots,n \label{eq:pp:2} \\
    (y-1) p_j & = 0 \qquad j = 1,\dots,m \label{eq:pp:3}
\end{align}

\begin{lemma}
\label{lem:extra_solution}
The system above has the following solutions:
\begin{enumerate}
\item $y = 1$, $x_1 = \dots = x_n = 0$
\item $y = 0$, $x_1 = \xi_1,\dots,x_n = \xi_n$, where $\xi \in \IR^n$ is any solution to the original instance. 
\end{enumerate}
In particular, the system always has a solution. It has more than one solution iff the original
$\QUAD$-instance is satisfiable. 
\end{lemma}
\begin{proof}
The first equation (\ref{eq:pp:1} constrains $y$ to be $\{0,1\}$-valued.
If $y = 0$, then the equations (\ref{eq:pp:2}) are trivially satisfied and
(\ref{eq:pp:3} reduces to the original instance (\ref{eq:Poly}).
If $y = 1$, then the equations (\ref{eq:pp:3}) are trivially satisfied and
(\ref{eq:pp:2} reduces to $x_1 = \dots = x_n = 1$. Note that in both cases
we always get different solutions since the $y$-value differs.
\end{proof}

Using the transformation in the lemma above, we can map any $\QUAD$-instance
into a $\QUADpp$-instance and obtain

\begin{corollary}
    \label{col:quadpp_hardness}
    \ETRpp and \QUADpp are $\existsR$-hard.  
\end{corollary}

\section{Hardness of Numerical Identifiability}

This section is dedicated to proving:
\begin{theorem}
    \label{thm:num_ident_hardness}
    Numerical identifiability is $\forallR$-hard.
\end{theorem}

The proof consists of building a polynomial-time reduction from the complement of $\QUADpp$ to numerical identifiability, i.e.,\ we construct an acyclic mixed graph $G$ and a $\Sigma \in \im \phi_G$, such that the fiber $\fib_G(\Sigma)$ has size $1$ iff the given $\QUADpp$-instance has only one solution.
For this, we use the following characterization of fibers due to \cite{drton2018algebraic}:
\begin{lemma}
    \label{lem:drton_fiber_iso}
    Let $G = (V, D, B)$ be an acyclic mixed graph, and let $\Sigma \in \im \phi_G$.
    The fiber $\fib_G(\Sigma)$ is isomorphic to the set of matrices $\Lambda \in \IR^D$ that solve the equation system
    \begin{align}
        [(I - \Lambda)^T\Sigma(I - \Lambda)]_{i,j} &= 0, && i \neq j, i \leftrightarrow j \notin B
    \end{align}
\end{lemma}
We construct $G$ as follows:
The directed edges form a bipartite graph with edges going from the bottom layer to the top layer.
Every node at the bottom layer has outdegree one.
Moreover bidirected edges exist between all pairs of nodes, except certain pairs of nodes of the top layer.
See Figure~\ref{fig:missing_edge_and_variable_gadget} for an illustration.

\begin{figure}[h]
    \vspace{-1.5em}
    \centering
    \begin{tikzpicture}[scale=0.65]
    \footnotesize
        \node[circle, draw, minimum size=21pt, scale=0.9] (i) at (0,0) {$i$};
        \node[circle, draw, minimum size=21pt, scale=0.9] (a1) at (-1,-2) {$a_1$};
        \node[minimum size=21pt, scale=0.9] (adots) at (0,-2) {$\dots$};
        \node[circle, draw, minimum size=21pt, scale=0.9] (an) at (1,-2) {$a_n$};
        \node[circle, draw, minimum size=21pt, scale=0.9] (j) at (4,0) {$j$};
        \node[circle, draw, minimum size=21pt, scale=0.9] (b1) at (3,-2) {$b_1$};
        \node[minimum size=21pt, scale=0.9] (bdots) at (4,-2) {$\dots$};
        \node[circle, draw, minimum size=21pt, scale=0.9] (bm) at (5,-2) {$b_m$};
        \draw[->] (a1) -- node [left] {$\lambda_{a_1,r}$} (i);
        \draw[->] (an) -- node [right] {$\lambda_{a_n,r}$} (i);
        \draw[->] (b1) -- node [left] {$\lambda_{b_1,r}$} (j);
        \draw[->] (bm) -- node [right] {$\lambda_{b_m,r}$} (j);
        \draw[color = blue] (i) -- node [above] {missing} (j);
    \end{tikzpicture}
    \hspace{6em}
    \begin{tikzpicture}[scale=0.65]
    \footnotesize
        \node[circle, draw, minimum size=21pt, scale=0.9] (r) at (0,0) {$r$};
        \node[circle, draw, minimum size=21pt, scale=0.9] (1) at (-1,-2) {$1$};
        \node[minimum size=21pt, scale=0.9] (dots) at (0,-2) {$\dots$};
        \node[circle, draw, minimum size=21pt, scale=0.9] (n) at (1,-2) {$n$};
        \draw[->] (1) -- node [left] {$\lambda_{1,r}$} (r);
        \draw[->] (n) -- node [right] {$\lambda_{n,r}$} (r);
    \end{tikzpicture}
    \caption{Left: A single missing edge on the top layer. Right: The gadget storing the value of each variable. $\lambda_{i,r}$ corresponds to $x_i$.}
    \label{fig:missing_edge_and_variable_gadget}
\end{figure}
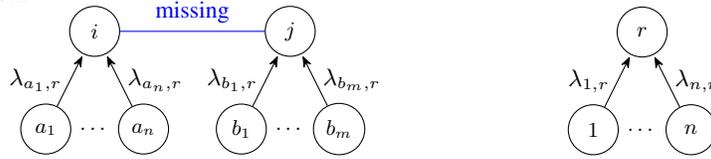

This missing edge in Figure~\ref{fig:missing_edge_and_variable_gadget} corresponds to the equation
\begin{equation} \label{eq:general_main}
  0 = \sigma_{i, j} - \sum_{\ell=1}^{n}\sigma_{a_\ell, j}\lambda_{a_\ell, i} - \sum_{k=1}^{m}\sigma_{b_k, j}\lambda_{a_\ell, i} + \sum_{\ell=1}^{n}\sum_{k=1}^{m}\sigma_{a_\ell, b_k}\lambda_{a_\ell, i}\lambda_{b_k, j}
\end{equation}
in Lemma~\ref{lem:drton_fiber_iso}.

\begin{observation}
    \label{obs:sigma_unique}
    All $\sigma$ values that appear in \eqref{eq:general_main} cannot appear in any other missing edge equation of missing edges in the top layer, since the nodes in the bottom layer only have outdegree one.
    Furthermore $\sigma_{i,j}$ can obviously only appear in this equation.
\end{observation}

The above observation means that we can freely ``program'' the equations.
We start with a gadget with one node $r$ in the top layer and $n$ nodes in the bottom layer connected to it.
It is used to store the value of each variable of our ETR instance, $\lambda_{1, r}$ corresponds to $x_1$, $\lambda_{2, r}$ corresponds to $x_2$, etc, see Figure~\ref{fig:missing_edge_and_variable_gadget} for an illustration.

By assuming all polynomials in our $\QUADpp$-instance are of the forms \eqref{eq:tseitin}, we need to be able to encode products and affine linear forms.
We start by showing how to encode an arbitrary affine linear constraint $\sum_{\ell=1}^{n}\alpha_\ell x_\ell = \beta$ using a single additional node $i$ in the top layer, ``connected'' to $r$ via a missing edge as in Figure~\ref{fig:addition_and_multiplication_gadget}.

\begin{figure}[h]
    \vspace{-2em}
    \centering
    \begin{tikzpicture}[scale=0.65]
    \footnotesize
        \node[circle, draw, minimum size=21pt, scale=0.9] (r) at (0,0) {$r$};
        \node[circle, draw, minimum size=21pt, scale=0.9] (i) at (3,0) {$i$};
        \node[circle, draw, minimum size=21pt, scale=0.9] (1) at (-1,-2) {$1$};
        \node[minimum size=21pt, scale=0.9] (dots) at (0,-2) {$\dots$};
        \node[circle, draw, minimum size=21pt, scale=0.9] (n) at (1,-2) {$n$};
        \draw[->] (1) -- node [left] {$\lambda_{1,r}$} (r);
        \draw[->] (n) -- node [right] {$\lambda_{n,r}$} (r);
        %\draw[->] (p) -- node [right] {$\lambda_{p,j}$} (j);
        \draw[color = blue] (r) -- node [above] {missing} (i);
    \end{tikzpicture}
    \hspace{8em}
    \begin{tikzpicture}[scale=0.65]
    \footnotesize
        \node[circle, draw, minimum size=21pt, scale=0.9] (r) at (0,0) {$r$};
        \node[circle, draw, minimum size=21pt, scale=0.9] (i) at (3,0) {$i$};
        \node[circle, draw, minimum size=21pt, scale=0.9] (j) at (6,0) {$j$};
        \node[circle, draw, minimum size=21pt, scale=0.9] (1) at (-1,-2) {$1$};
        \node[minimum size=21pt, scale=0.9] (dots) at (0,-2) {$\dots$};
        \node[circle, draw, minimum size=21pt, scale=0.9] (n) at (1,-2) {$n$};
        
        \node[circle, draw, minimum size=21pt, scale=0.9] (ip) at (3,-2) {$i'$};
        \node[circle, draw, minimum size=21pt, scale=0.9] (jp) at (6,-2) {$j'$};
        
        \draw[->] (1) -- node [left] {$\lambda_{1,r}$} (r);
        \draw[->] (n) -- node [right] {$\lambda_{n,r}$} (r);
        \draw[->] (ip) -- node [left] {$\lambda_{i',i}$} (i);
        \draw[->] (jp) -- node [left] {$\lambda_{j',j}$} (j);
        
        \draw[color = blue] (r) -- node [above] {missing} (i);
        \draw[color = blue] (i) -- node [above] {missing} (j);
        \draw[color = blue] (r) to[out=37,in=143] node [above] {missing} (j);
    \end{tikzpicture}
    \caption{Left: Gadget for affine linear constraints. Right: Gadget for multiplicative constraints}
    \label{fig:addition_and_multiplication_gadget}
\end{figure}
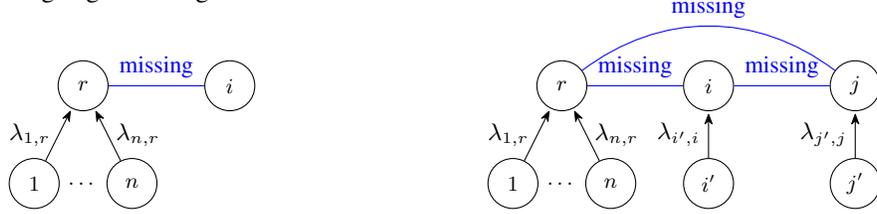

Setting $\sigma_{r, i} = \beta$ and $\sigma_{\ell,i} = \alpha_\ell$ makes $\eqref{eq:general_main}$ together with $\lambda_{\ell, r} = x_\ell$ directly equivalent to $\sum_{\ell=1}^{n}\alpha_\ell x_\ell = \beta$.

Encoding a product $x_a = x_b \cdot x_c$ requires two additional nodes $i$ and $j$ in the top layer, with missing bidirectional edges between them and $r$.
Furthermore we introduce two nodes $i'$ and $j'$ in the bottom layer, connected to $i$ and $j$ respectively, 
see Figure~\ref{fig:addition_and_multiplication_gadget}.
Setting $\sigma_{r, i} = \sigma_{1, i'} = \ldots = \sigma_{n, i'} = 0$, $\sigma_{r, i'} = -1$, $\sigma_{c, i} = 1$, and $\sigma_{\ell, i} = 0$, for all $\ell \in \{1, \ldots, n\} \setminus \{c\}$, ensures $\lambda_{i',i} = \lambda_{c, r}$.
To further ensure $\lambda_{j',j} = \lambda_{i',i} = \lambda_{c, r}$, we set $\sigma_{i, j} = \sigma_{i', j'} = 0$, $\sigma_{i', j} = 1$, and $\sigma_{i, j'} = -1$.
After having copied $\lambda_{c, r}$ twice, we are finally able to enforce the multiplication itself via $\sigma_{r, j} = \sigma_{r, j'} = 0$, $\sigma_{a, j} = 1$, $\sigma_{b, j'} = 1$, $\sigma_{\ell, j} = 0$, for all $\ell \in \{1, \ldots, n\} \setminus \{a\}$, and $\sigma_{\ell, j'} = 0$, for all $\ell \in \{1, \ldots, n\} \setminus \{b\}$.

\begin{proof}[Proof of Theorem~\ref{thm:num_ident_hardness}]
    Let polynomials $p_1, \ldots, p_m$ in variables $x_1, \ldots, x_n$ be a $\QUADpp$-instance with all polynomials being one of the forms in \eqref{eq:tseitin}.
    Let the number of affine linear polynomials among $p_1, \ldots, p_m$ be $k$.
    Then the graph $G = (V, D, B)$ constructed above has $\ell := 1+n+k+4(m-k) = 1+n+4m-3k$ nodes.
    Using Observation~\ref{obs:sigma_unique}, we see that the construction induces a well-defined partial matrix $\Sigma \in \IR^{\ell \times \ell}$.
    Every entry of $\Sigma$ not defined by the construction is set to $0$ if it is off-diagonal and $\ell$ if it is on the diagonal.
    Since all $\sigma_{ij}$ set in the construction are from $\{-1, 0, 1\}$ and off-diagonal, this makes $\Sigma$ strictly diagonally dominant and thus positive definite by the Gershgorin circle theorem \cite{gershgorin1931uber}.

    Remains to prove $\Sigma \in \im \phi_G$.
    Let $\xi \in \IR^n$ be any solution with $p_1(\xi) = \cdots = p_m(\xi) = 0$.
    The existence of $\xi$ is guaranteed by the promise of $\QUADpp$.
    Create $\Lambda \in \IR^{\ell \times \ell}$ as follows: $\lambda_{i,r} = \xi_i$ for $i \in \{1, \ldots, n\}$ and $\lambda_{i', i} = \lambda_{j', j} = \xi_c$ whenever the vertices $i, i', j, j'$ are the vertices added by the construction due to a multiplication.
    All other entries of $\Lambda$ are $0$.
    Then $I - \Lambda$ is invertible and we have $\Sigma = \phi_G(\Lambda, \Omega)$ for $\Omega = (I - \Lambda)^T\Sigma(I - \Lambda)$.
    Furthermore $\Omega$ is positive definite due to $\Sigma$ being positive definite and $\Omega \in \PD(B)$.

    By Lemma~\ref{lem:drton_fiber_iso}, this implies that $|\fib_G(\Sigma)|$ is precisely the number of solutions of our $\QUADpp$-instance and the theorem now follows from the the $\forallR$--hardness of the complement of $\QUADpp$ which is equivalent to Corollary~\ref{col:quadpp_hardness}.
\end{proof}

\section{Upper Bound for Numerical Identifiability}

%\subsection{Global Identifiability}

%In this section we show a $\forallR$ upper bound for global identification.
%This then proves a $\PSPACE$ upper bound using the classical quantifier elimination result in %\cite{renegar1992computational}.

%\medskip

%Rough formula:
%\begin{align*}
%    &\forall \Lambda_1, \Lambda_2 \in \IR^D, \Omega_1, \Omega_2 \in \PD(B): \forall A_1, A_2 \in %\IR^{m \times m}:\\
%    &(A_1 \cdot (I - \Lambda_1) = I \land A_2 \cdot (I - \Lambda_2) = I \land A_1^T\Omega_1A_1 = A_2^T\Omega_2A_2) \implies (\Lambda_1 = \Lambda_2 \land \Omega_1 = \Omega_2)
%\end{align*}

In this section, we show a $\forallR$ upper bound for numerical identifiability and thus, combined with Theorem~\ref{thm:num_ident_hardness}, prove an almost\footnote{see Remark~\ref{rem:promise} for the details.	} matching lower and upper bound.
We start with the following lemma. The $\existsR$ part will be needed in the next section.

\begin{lemma}
    \label{lem:pd_membership}
    Membership in $\PD(n)$ and $\PD(B)$ can be expressed in $\existsR$ and $\forallR$.
    \footnote{Note that membership in $\PD(n)$ can be even decided faster. However, this will not change the complexity of our overall algorithm.}
\end{lemma}
\begin{proof}
    For the $\existsR$ expression, we use the fact that every real positive definite matrix $A \in \IR^{n \times n}$ has a Cholesky decomposition $A = L L^T$ where $L$ is a real lower triangular matrix with positive diagonal entries.
    We can thus express $A \in \PD(n)$ as
    \begin{align} \label{eq:gi:1}
        \exists L \in \IR^{n \times n}: A = L L^T \land \bigwedge_{i \in \{1, \ldots, n\}} (L_{i,i} > 0 \quad \land \bigwedge_{j \in \{i+1, \ldots, n\}} L_{i,j} = 0).
    \end{align}
    We quantify over matrices and consider matrix equations in (\ref{eq:gi:1}). But this can be easily rewritten as an $\ETR$-instance by quantifying over all entries of the matrix and having one individual equation for each entry of the matrix equation.

    For the $\forallR$ expression, we directly use the definition of positive definite matrices to express $A \in \PD(n)$ as
    \begin{align} \label{eq:gi:2}
        \forall x \in \IR^n: x \neq 0 \implies x^T A x > 0.
    \end{align}

    For membership $A \in \PD(B)$, in both $\existsR$ and $\forallR$, we add the constraint
    %\begin{align}
        $\bigwedge_{(i,j) \notin B \land i \neq j} A_{i,j} = 0$
    %\end{align}
    to (\ref{eq:gi:1}) and (\ref{eq:gi:2}), respectively.
\end{proof}

We remind the reader that numerical identifiability is a promise problem with the promise that the input $\Sigma \in \im \phi_G$, so it suffices to check whether all elements in the fiber $\fib_G(\Sigma)$ are identical.
\begin{align}
    \forall \Lambda_1, \Lambda_2 \in \IR^D, &\Omega_1, \Omega_2 \in \PD(B): \nonumber\\
    &\phi_G(\Lambda_1, \Omega_1) = \phi_G(\Lambda_2, \Omega_2) = \Sigma \Rightarrow (\Lambda_1 = \Lambda_2 \land \Omega_1 = \Omega_2). \label{eq:num_ident_forallr_formula}
\end{align}
The checks $\phi_G(\Lambda_i, \Omega_i) = \Sigma$ are implemented using $\Omega_i = (I-\Lambda_i)^T\Sigma(I-\Lambda_i)$, which is equivalent due to $I-\Lambda_i$ being invertible for any $\Lambda_i \in \IR^D$. This proves the following:

\begin{theorem}
    Numerical identifiability is in (the promise version of) $\forallR$.
\end{theorem}

\begin{remark}
    \label{rem:promise}
    Strictly speaking, numerical identifiability is not contained in $\forallR$, since it is a promise problem, that is, the outcome is not specified for $\Sigma$ that are not feasible. $\forallR$ consists by definition only of classical decision problems, where the outcome is specified for \emph{all} inputs. So the corresponding 
    complexity class is $\mathrm{Promise}\text{-}\forallR$. Section~\ref{sec:promise} contains some more information on promise problems for the reader's convenience.
\end{remark}

However, we can express feasibility in $\existsR$:
\begin{lemma} \label{lem:feas}
    \label{lem:im_phi_G_membership}
    Membership in $\im \phi_G$ can be expressed in $\existsR$.
\end{lemma}
\begin{proof}
    We use the expression
    %\begin{align*}
        $\exists \Lambda \in \IR^D, \Omega \in \PD(B): (I - \Lambda)^T \Sigma (I - \Lambda) = \Omega$,
    %\end{align*}
    where we use Lemma~\ref{lem:pd_membership} to express $\Omega \in \PD(B)$ in $\existsR$, that is, we quantify over an arbitrary matrix $\Omega$ first and add the ETR expression from 
    Lemma~\ref{lem:pd_membership} to ensure that $\Omega$ is in $\PD(B)$.
\end{proof}
So we can check in $\existsR$ whether the input $\Sigma$ is feasible and then in $\forallR$ whether the fiber has only one element. Using Renegar's algorithm, we get:

\begin{corollary}
    \label{cor:num_ident_pspace}
    Numerical identifiability can be decided in polynomial space.
\end{corollary}

\section{Generic Identifiability is in PSPACE}
\label{sec:gen_ident_in_pspace}
Let $\DIM$ denote the following problem: Given an encoding of a semi-algebraic set $S$
and a number $d$, decide whether $\dim S \ge d$. 

\begin{lemma}[\cite{DBLP:journals/jc/Koiran99}\footnote{see Appendix~\ref{sec:koiran} for why this statement follows from \cite{DBLP:journals/jc/Koiran99}.}]
The problem $\DIM$ is $\existsR$-complete. Moreover, this is even true when the set is given by an existentially quantified formula as in (\ref{eq:etr:1}).    
\end{lemma}

We use the same notation as in Section~\ref{sec:ident}. Let $G = (V,D,B)$ be a mixed graph.
Let $S_G = \{ (\Lambda, \Omega) \mid |\fib_G(\Lambda, \Omega)| > 1, \Lambda \in \IR^D, \Omega \in \PD(B)  \}$.

\begin{observation} \label{obs:key}
$G$ is generically identifiable iff $\dim \IR^D \cdot \dim \PD(B) > \dim S_G$.
\end{observation}
\begin{proof}
As $S_G \subseteq \IR^D \times \PD(B)$ and $\dim (\IR^D \times \PD(B)) = \dim \IR^D \cdot \dim \PD(B)$, the righthand side is just the definition of being generically identifiable.
%As $S_G \subseteq \Theta$ and $\dim \Theta = \dim \IR^D \cdot \dim \PD(B)$, the righthand side is just the definition of being generically identifiable.
\end{proof}

We postpone the proof that membership in $S_G$ can be expressed in $\existsR$, in favor of first giving our algorithm to decide generic identifiability, using this observation:
\begin{theorem}\label{thm:generic:identification}
Generic identifiability is both in $\forall\exists\IR$ and $\exists\forall\IR$.
\end{theorem}
\begin{proof}
Let $G$ be the given mixed graph.
Formulate membership in $\IR^D, \PD(B)$, and $S_G$ as instances of $\ETR$ using Lemmas~\ref{lem:pd_membership} and~\ref{lem:S_G_membership}.
Note that the number of variables and the sizes of these instances are polynomial in the size of $G$.

We first assume that we have oracle access to $\ETR$, that is, we can query $\ETR$ a polynomial number of times.
We decide whether $G$ is generically identifiable as follows:
\begin{enumerate}
   \item Use Koiran's algorithm repeatedly to compute $\dim S_G$ by checking whether $\dim S_G \ge d$ for $d = 0,\dots,m^4$. (We could even use binary search.)
   \item Compute $\dim \IR^D$ and $\dim \PD(B)$ in the same way.\footnote{While we could compute these dimensions more directly, this is not necessary for our specific algorithm. However for implementing this algorithm in practice, we would advise computing them more efficiently}
   Here it suffices to check up to the maximum possible dimension of $d = m^2$.
   \item Accept if $\dim \IR^D \cdot \dim \PD(B) > \dim S_G$, and reject otherwise.
\end{enumerate}
The algorithm is correct by Observation~\ref{obs:key}.
The algorithm above would already show the $\PSPACE$ upper bound for generic identifiability.

However, we can implement the algorithm above by a single formula by replacing the repeated use of Koiran's algorithm by a big disjunction:
\[
    \bigvee_{d_1=0}^{m^2}\bigvee_{d_2=0}^{m^2} (\dim \IR^D \geq d_1 \land \dim \PD(B) \geq d_2 \land \dim S_G < d_1d_2)\,.
\]

Note that the check $\dim S_G < d_1d_2$ needs to be implemented as $\lnot (\dim S_G \geq d_1d_2)$, thus being in $\forallR$ by De Morgan's laws.
The existential and universal quantifiers are however independent, giving upper bounds of both $\forall\exists\IR$ and $\exists\forall\IR$.
\end{proof}
Using Renegar's algorithm this implies:
\begin{corollary} \label{cor:gen_ident:pspace}
    Generic identifiability can be decided in $\PSPACE$.   
\end{corollary}

Only remains to show how to express membership in $S_G$ as an $\ETR$-formula.

\begin{lemma}\label{lem:S_G_membership}
    Membership in $S_G$ can be expressed in $\existsR$.
\end{lemma}
\begin{proof}
    The membership of some $(\Lambda, \Omega)$ in $S_G$ can be expressed as
    \begin{align*}
        \Lambda \in \IR^D \land \Omega \in \PD(B) \land \exists \Sigma \in \IR^{n \times n}, \Lambda' \in \IR^D, \Omega' \in \PD(B)&: (I - \Lambda)^T \Sigma (I - \Lambda) = \Omega \nonumber\\
        &\quad \land (I - \Lambda')^T \Sigma (I - \Lambda') = \Omega' \nonumber\\
        &\quad \land (\Lambda \neq \Lambda' \lor \Omega \neq \Omega')\,.\qedhere
    \end{align*} 
\end{proof}

\section{A note on cyclic graphs}

\newcommand{\IRreg}{\IR_{\mathrm{reg}}}
%\color{red}
Our results so far have depended on the fact that every matrix $I-\Lambda$ 
%If a graph is cyclic, not every matrix $I-\Lambda$ 
with $\Lambda\in\IR^D$ is invertible if the graph is acyclic. 
However, if the graph is cyclic, $I-\Lambda$ is not necessarily invertible.
So in this case, we need to explicitly consider the subset $\IRreg^D$ of matrices $\Lambda\in \IR^D$ such that $I-\Lambda$ is invertible.

For matrices $\Lambda_0 \in \IRreg^D$ and $\Omega_0 \in \PD(B)$, \cite{foygel2012half} define fibers as
\[
   \fib_G(\Lambda_0,\Omega_0) = \{ (\Lambda,\Omega) \mid 
   \phi_G(\Lambda,\Omega) = \phi_G(\Lambda_0,\Omega_0),  
   \Lambda \in \IRreg^D, \Omega \in \PD(B) \}.
\]
They define generic identifiability for \emph{cyclic} graphs in terms of these fibers. 
That is a mixed (cyclic) graph $G$ is said to be \emph{generically} identifiable if
% $\phi_G$ is injective on the complement $\Theta \setminus {\cal V}$ of a proper (i.e., strict) 
% algebraic subset ${\cal V}\subset \Theta$.
%(i.e. $\phi_G$ injectiv on the complement of a strict algebraic subset, .. $\fib_G(\Lambda_0,\Omega_0) = 1$.. )
%$G$ is generically identifiable, if 
$|\fib_G(\Lambda_0,\Omega_0)| = 1$ for Zariski almost all    $\Lambda_0 \in \IRreg^D$ and $\Omega_0 \in \PD(B)$.

This is the same criterion used for acyclic graphs, except $\IR^D$ has been replaced by $\IRreg^D$ twice. 
Matrix invertibility can be easily expressed in $\existsR$, using the definition of invertibility that there exists a matrix whose product leads to the identity matrix.

Hence, all our results also hold for cyclic graphs.

\section{Conclusion}
Due to double exponential runtime, the state-of-the-art algorithm for the generic identification problem is often too slow to solve instances of reasonable size.
For example, García-Puente et al.~ \cite{garcia2010identifying}
note the runtime varies between seconds and 75 days for graphs with four nodes.
An interesting topic for future work would be to implement the (theoretical) algorithm 
presented in our paper.

We have given a new upper bound on the complexity of generic identifiability, namely $\exists\forall\IR$ and $\forall\exists\IR$.
Can this be further improved, for example to $\existsR$ or $\forallR$? In the light of our hardness proofs for the new notion of numerical identifiability, we conjecture that generic identifiability is hard for either $\existsR$ or $\forallR$.

\bibliographystyle{plain}
\bibliography{lit}

\appendix

\section*{Appendix}

\section{Edge Identifiability}
\label{sec:edge}
For edge identifiability, we obtain the same results as for identifiability itself.
\subsection{Hardness of Numerical Edge Identifiability}
If we analyze the reduction in Theorem~\ref{thm:num_ident_hardness} in a bit more detail, we can use the same reduction to show
\begin{corollary}
    Numerical edge identifiability is $\forallR$-hard.
\end{corollary}
\begin{proof}
    If instead of starting with an arbitrary $\QUADpp$-instance, we start with a $\QUADpp$-instance generated by Lemma~\ref{lem:extra_solution} and Corollary~\ref{col:quadpp_hardness}.
    These instances have a distinguished variable $y$, such that there is always a single solution with $y=1$ and possibly multiple solutions with $y=0$.
    W.l.o.g.\ let $x_1$ be this distinguished variable.
    Following the reduction in Theorem~\ref{thm:num_ident_hardness} we construct a graph $G$ and a $\Sigma \in \im \phi_G$, such that the fiber $\fib_G(\Sigma)$ is isomorphic to the solutions our $\QUADpp$-instance.
    In particular the value of $\lambda_{1,r}$ in the elements of $\fib_G(\Sigma)$ is exactly the value of $x_1$ in solutions to the $\QUADpp$-instance.
    Thus $|\fib^{1,r}_G(\Sigma)| = 1$ iff the $\QUADpp$ instance has exactly one solution, otherwise $|\fib^{1,r}_G(\Sigma)| = 2$.
\end{proof}

\subsection{Upper Bound for Numerical Edge Identifiability}
Similarly to \eqref{eq:num_ident_forallr_formula}, we can express numerical edge identifiability as the $\forallR$-formula
\begin{align}
    \forall \Lambda_1, \Lambda_2 \in \IR^D, &\Omega_1, \Omega_2 \in \PD(B): \nonumber\\
    &\phi_G(\Lambda_1, \Omega_1) = \phi_G(\Lambda_2, \Omega_2) = \Sigma \Rightarrow (\Lambda_1)_{i,j} = (\Lambda_2)_{i,j}\,.
\end{align}
This yields
\begin{theorem}
    Numerical edge identifiability is in (the promise version of) $\forallR$.
\end{theorem}
Again using Renegar's algorithm we also get
\begin{corollary}
    Numerical edge identifiability can be decided in polynomial space.
\end{corollary}

\subsection{Generic Edge Identifiability is in PSPACE}
We modify the algorithm of Section~\ref{sec:gen_ident_in_pspace} to work with generic edge identifiability for some $\lambda_{i,j}$ rather than generic identifiability.
Let $S^{i,j}_G = \{ (\Lambda, \Omega) \mid |\fib^{i,j}_G(\Lambda, \Omega)| > 1, \Lambda \in \IR^D, \Omega \in \PD(B)  \}$.
Then we have the following analog to Observation~\ref{obs:key}:
\begin{observation} \label{obs:key:edge}
$\lambda_{i,j}$ is generically edge identifiable iff $\dim \IR^D \cdot \dim \PD(B) > \dim S^{i,j}_G$.
\end{observation}

\begin{lemma}
    Membership in $S^{ij}_G$ can be expressed in $\existsR$.
\end{lemma}
\begin{proof}
    We use a similar formula to Lemma~\ref{lem:S_G_membership}:
    \begin{align*}
        \Lambda \in \IR^D \land \Omega \in \PD(B) \land \exists \Sigma \in \IR^{m \times m}, \Lambda' \in \IR^D, \Omega' \in \PD(B)&: (I - \Lambda)^T \Sigma (I - \Lambda) = \Omega \nonumber\\
        &\quad \land (I - \Lambda')^T \Sigma (I - \Lambda') = \Omega' \nonumber\\
        &\quad \land (\Lambda_{i,j} \neq \Lambda'_{i,j})
    \end{align*}
\end{proof}

\begin{theorem}
Generic edge identifiability is both in $\forall\exists\IR$ and $\exists\forall\IR$.
\end{theorem}
\begin{proof}
We use the algorithm of Theorem~\ref{thm:generic:identification}, but replace the set $S_G$ by $S^{i,j}_G$.
\end{proof}
\begin{corollary}
 Generic edge identifiability can be decided in $\PSPACE$.   
\end{corollary}

\newcommand{\SAT}{\mathrm{SAT}}
\newcommand{\SATpp}{\mathrm{SAT}^{++}}

\section{Promise problems}
\label{sec:promise}

We give some background information on promise problems for the readers convenience.
In a classical decision problem $L$, we are given an input an we have decide whether $x\in L$
(the so-called yes-instances) or $x \notin L$ (the no-instances). For instance, in the classical 
$\SAT$ problem, we are given a formula $F$ in CNF. The yes-instances are the satisfiable formulas and the no-instances are the unsatisfiable one.

Promise problems have a third type of instances, the so-called do-not-care-instances. On these instances, an algorithm can do what it wants and give any output. For instance, consider the problem $\SATpp$, where ask the question of whether a satisfiable formula in CNF has another satisfying assignment. The yes-instances are all $F$ with at least two satisfying assignments, the no-instances are all $F$ with exactly one satisfying assignment, and the do-not-care-instances are all unsatisfiable $F$. An algorithm solving $\SATpp$ has to output ``yes'' on every $F$
with at least two satisfying assignments and ``no'' on every $F$ with exactly one satisfying assignment. On unsatisfiable formulas, it can output whatever it wants. Note that every classical decision problem is also a promise problem with the do-not-care-instances being the empty set.

We can also define many-one reductions for promise problems: A function $f: \{0,1\}^* \to \{0,1\}^*$ is called a many-one reduction from a promise problem $L$ to another promise problem $L'$, if $f$ maps yes-instances of $L$ to yes-instances of $L'$ and no-instances of $L$ to no-instances of $L'$. $f$ can map do-not-care-instances of $L$ to any instance of $L'$. By using a similar trick of encoding an additional satisfying assignment like in the case of $\ETRpp$, one can show that 
$\SATpp$ is $\NP$-hard, since we can reduce $\SAT$ to it. This reduction maps the unsatisfiable formulas (no-instances of $\SAT$) to formulas with one satisfying assignment (no-instances of $\SATpp$) and  satisfiable formulas (yes-instances of $\SAT$) to formulas with two or more satisfying assignments (yes-instances of $\SATpp$). Since $\SAT$ is a classical decision problem, there are no do-not-care-instances. $\SATpp$ is, however, not contained in $\NP$ for formal reasons, because $\NP$ only contains classical decision problems.

\section{Semialgebraic sets}

For the reader's convenience, we give a brief introduction to semialgebraic sets and discuss the notations important for this work. For details and proofs, we refer to the book \cite{real}.

A \emph{semialgebraic set} in $\IR^n$ is a finite Boolean combination (finite number of unions and intersections) of sets of the form $\{(x_1,\dots,x_n) \mid f(x_1,\dots,x_n) > 0\}$
and $\{(x_1,\dots,x_n) \mid g(x_1,\dots,x_n) \geq 0\}$. Here $f$ and $g$ are real polynomials in $n$ variables.

A \emph{semialgebraic function} is a function $\IR^n \to \IR^{n'}$ 
with a semialgebraic graph, that is, the
set of all $\{(x,f(x)) \mid x \in \IR^n \}$ is a semialgebraic set.

From this definition of semialgebraic sets, it is easy to see that semialgebraic sets are the solutions of \ETR-instances, that is, all $(x_1,\dots,x_n)$ satisfying $\varphi(x_1,\dots,x_n)$ in
(\ref{eq:etr:1}) form a semialgebraic set. From Tarski's theorem (see \citep{real}), it follows that semialgebraic sets allow quantifier elimination, that is, all $(x_1,\dots,x_n)$
satisfying
\[
   \exists y_1 \dots \exists y_t \psi(x_1,\dots,x_n,y_1,\dots,y_t)
\]
form a semialgebraic set, where  $\psi$ (like $\varphi$)  is a quantifier-free Boolean formula over the basis $\{\vee, \wedge, \neg\}$ and a signature consisting of the constants $0$ and $1$, the functional symbols $+$ and $\cdot$, and the relational symbols $<$, $\le$, and $=$. $\psi$ depends on two sets of variables. It is clear that all $(x_1,\dots,x_n,y_1,\dots,y_t)$ satisfying $\psi$ form a semialgebraic set. Tarski's theorem tells us that we still get a semialgebraic set when
we are projecting the $y_1,\dots,y_t$ away using the existential quantifiers. This is used frequently in our proofs.

\begin{definition} \label{def:semi:dim}
A semialgebraic set $S$ has dimension $d$ if there exists a
$d$-dimensional coordinate subspace such that the image of $S$ under the canonical projection
onto this subspace has a nonempty interior and there is no such subspace of dimension $d+1$.
\end{definition}

Above, by a $d$-dimensional coordinate subspace, we mean the subspace of $\IR^n$ of all points
$x$ such that $x_i = 0$ for $i \notin I$ for some subset $I \subseteq \{1,\dots,n\}$ 
and $|I| = d$.

\begin{theorem}[\cite{bochnak2013real}]
Let $S$ be a semialgebraic set. Then its dimension as a semialgebraic set (as in Definition~\ref{def:semi:dim}) equals the dimension of its Zariski closure as an algebraic set.
\end{theorem}

\section{Koiran's algorithm}
\label{sec:koiran}
Koiran mainly works in the so-called BSS-model of real computation. In this model, one is also allowed to use arbitrary real constants in the algorithm as well as real inputs. In $\existsR$,
we only allow the constants $0$ and $1$ and the inputs are given by some binary encoding. However, Koiran also considers the bit model. \cite[Theorem 6]{DBLP:journals/jc/Koiran99} proves the computational equivalence of $\DIM$ and $\FFEAS$ in the bit model. Since $\FFEAS$
is $\existsR$-complete \citep{DBLP:journals/mst/SchaeferS17}, this implies that
$\DIM$ is $\existsR$-complete. Note that at the time Koiran wrote his paper,
the class $\existsR$ was not formally defined and therefore, Koiran does not mention it explicitly.

For the moreover part, note that \cite[Section 1.1]{DBLP:journals/jc/Koiran99} discusses the representations of semi-algebraic sets that are supported by his proof. There he mentions existentially quantified formulas explicitly.

\end{document}